\documentclass[letterpaper]{article} 
\usepackage[draft]{aaai25}  
\usepackage{times}  
\usepackage{helvet}  
\usepackage{courier}  
\usepackage[hyphens]{url}  
\usepackage{graphicx} 
\usepackage{natbib}  
\usepackage{caption} 
\usepackage{amssymb}
\usepackage{pifont}
\usepackage[ruled,vlined]{algorithm2e}

\usepackage{amsthm}
\usepackage{amsmath}
\usepackage{amsfonts} 
\usepackage{bm}
\usepackage{newfloat}
\usepackage{listings}
\usepackage{xcolor}
\usepackage{ctable} 
\usepackage{color, colortbl}

\newcommand{\dosemic}{\renewcommand{\@endalgocfline}{\algocf@endline}}
\let\oldnl\nl
\newcommand{\nonl}{\renewcommand{\nl}{\let\nl\oldnl}}

\newcommand{\bmx}{\mathbf{x}}
\newcommand{\bmz}{\mathbf{z}}

\newtheorem{theorem}{Theorem} 
\newtheorem*{theorem*}{Theorem}

\newtheorem{corollary}{Corollary}

\newtheorem{corollary*}{Corollary}

\DeclareCaptionStyle{ruled}{labelfont=normalfont,labelsep=colon,strut=off} 
\title{LoRID: Low-Rank Iterative Diffusion for Adversarial Purification}

\author{
Geigh Zollicoffer\textsuperscript{\rm 1}\equalcontrib,
    Minh Vu\textsuperscript{\rm 1}\equalcontrib,
    Ben Nebgen\textsuperscript{\rm 1},
    Juan Castorena\textsuperscript{\rm 2},    
    Boian Alexandrov\textsuperscript{\rm 1},
    Manish Bhattarai\textsuperscript{\rm 1}
    }
\affiliations{
    \textsuperscript{\rm 1} Theoretical Division, Los Alamos National Laboratory, Los Alamos, NM\\
    \textsuperscript{\rm 2} Computational Sciences, Los Alamos National Laboratory, Los Alamos, NM\\
    gzollicoffer@lanl.gov, mvu@lanl.gov, bnebgen@lanl.gov, jcastorena@lanl.gov, boian@lanl.gov, ceodspspectrum@lanl.gov
}

\begin{document}

\maketitle

\begin{abstract}
This work presents an information-theoretic examination of diffusion-based purification methods, the state-of-the-art adversarial defenses that utilize diffusion models to remove malicious perturbations in adversarial examples. By theoretically characterizing the inherent purification errors associated with the Markov-based diffusion purifications, we introduce LoRID, a novel Low-Rank Iterative Diffusion purification method designed to remove adversarial perturbation with low intrinsic purification errors. LoRID centers around a multi-stage purification process that leverages multiple rounds of diffusion-denoising loops at the early time-steps of the diffusion models, and the integration of Tucker decomposition, an extension of matrix factorization, to remove adversarial noise at high-noise regimes. Consequently, LoRID increases the effective diffusion time-steps and overcomes strong adversarial attacks, achieving superior robustness performance in CIFAR-10/100, CelebA-HQ, and ImageNet datasets under both white-box and black-box settings. 
\end{abstract}

%

\section{Introduction}
\label{sec:intro}

Despite their widespread adoption, neural networks are vulnerable to small malicious input perturbations, leading to unpredictable outputs, known as \textit{adversarial attacks}~\cite{szegedy2014intriguing, goodfellow2015explaining}. Various defense methods have been developed to protect these models~\cite{app9050909}, including \textit{adversarial training}~\cite{madry2019deep, bai2021recent, zhang2019theoretically} and \textit{adversarial purification}~\cite{salakhutdinov2015learning,shi2021online,song2018pixeldefend, nie2022DiffPure, wang2022guided, wang2023better}.
With the introduction of \textit{diffusion models}~\cite{ho2020denoising, song2021scorebased} as a powerful class of generative models, diffusion-based adversarial purifications have overcome training-based methods and achieve state-of-the-art (SOTA) robustness performance~\cite{blau2022threat, wang2022guided,nie2022DiffPure,xiao2022densepure}. In principle, the diffusion-based purification first diffuses the adversarial inputs with Gaussian noises in $t$ time-steps and utilizes the diffusion's denoiser to remove the adversarial perturbations along with the added Gaussian noises. 
While it is computationally challenging to attack diffusion-based purification due to vanishing/exploding gradient problems, high memory costs, and substantial randomness~\cite{kang2024diffattack}, recent work has been proposing efficient attacks against diffusion-based purification~\cite{nie2022DiffPure,kang2024diffattack}, which can degrade the model robustness significantly. A naive way to prevent such attacks is to increase the diffusion time-step $t$ as it will remarkably raise both the time and memory complexity for the attackers~\cite{kang2024diffattack}. However, increasing $t$ would not only introduces additional computational cost of purification~\cite{nie2022DiffPure, lee2023robust}, but also inevitably damages the purified samples (see Theorem~\ref{theorem:ddpm_time} or Fig.~\ref{fig:illu_purified}), and significantly degrade the classification accuracy.

\begin{table}[t]
\caption{Performance of SOTA score-based purification versus our proposed LoRID, a Markov-based purification, in {\color{black} \underline{CIFAR-10}} ($\epsilon = 8/255$) and {\color{black} ImageNet} ($\epsilon = 4/255$) under $L_\infty$ white-box AutoAttack in WideResNet-28-10.}
\label{table:score_vs_markov}
\vspace{-2mm}
\centering
\resizebox{0.99\columnwidth}{!}{%
\begin{tabular}{@{}ccc@{}}
\toprule 
\textbf{Purification} & \textbf{Score-based}                                  & \textbf{LoRID} \\ \midrule
Standard Acc     & \underline{ 89.02} / {\color{black} 71.16} &  \underline{ 84.20} / {\color{black} 73.98}      \\
Robust Acc  &    \underline{46.88} / {\color{black} 44.39}                                                   &         \underline{ 54.14} / {\color{black} 56.54}                     \\
Inference Run-time Speedup         &     $\underline{\times 1}$/  $\times 1$           &              $\underline{\times 2.3}$ / $\times 4.6$                  \\ \bottomrule
\end{tabular}
}
\vspace{-5mm}
\end{table}
\begin{figure}[h]
		\centering
\includegraphics[width=0.9\linewidth]{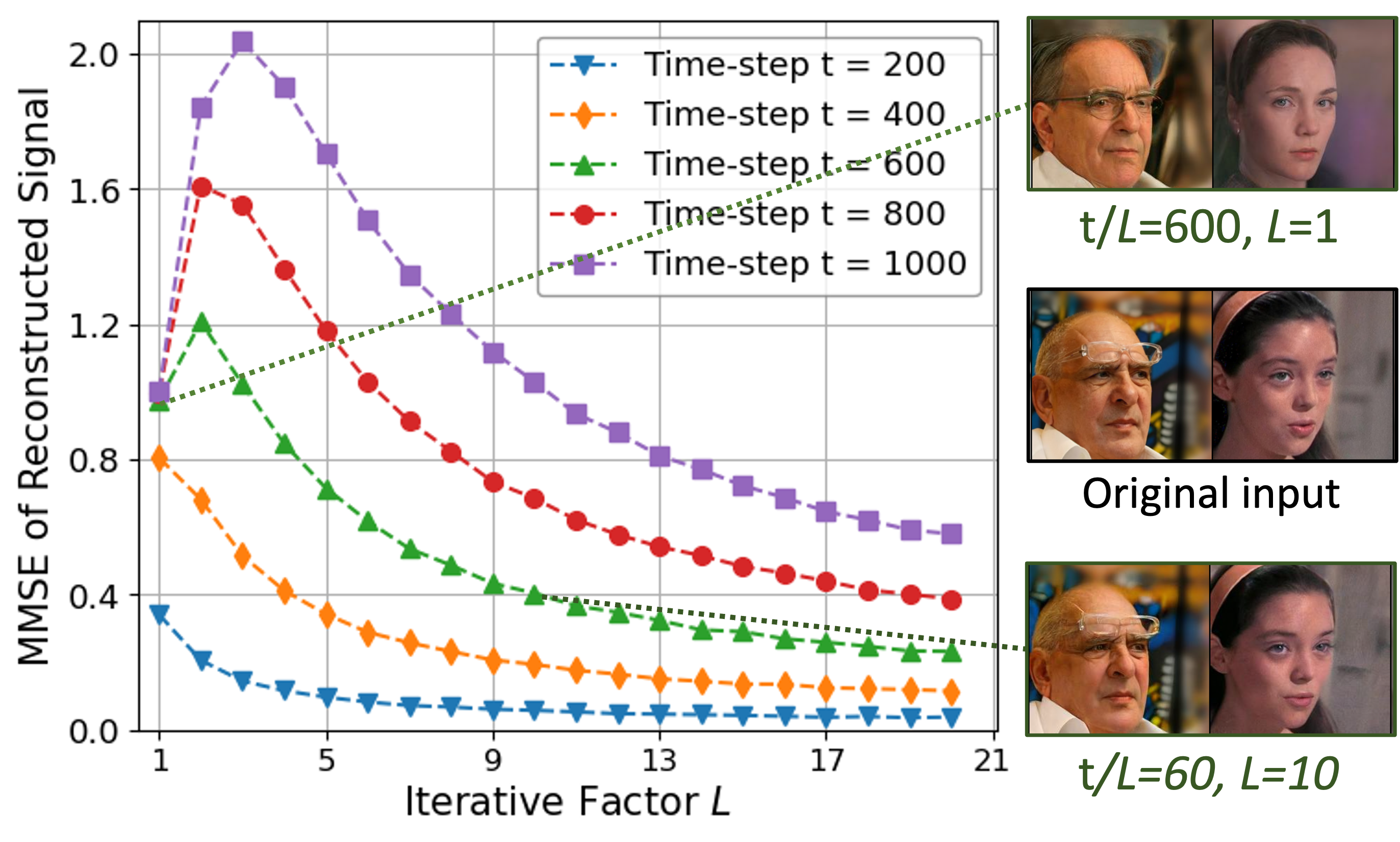}
		\caption{The MMSEs induced by Markov-based purification against the iterative factor $L$ (Corollary~\ref{corollary:clean}): each point is the MMSE of the reconstructed data from a normalized Gaussian through $L$ iterative loops of $t/L$ diffusion-denoising calls. Thus, points on a line share the same effective denoising step $t = (t/L) \times L$. The key observation is the purification error generally decreases as $L$ increases. The right samples compare clean samples, purified samples with a single large time-step $t/L=600$, and those with the same effective denoising step $t$ but with a larger iterative factor $L =10$ (Details in Appx.~\ref{appx:mmse}).}
		\label{fig:msse_vs_max_iter}
  
\end{figure}

\begin{figure*}[ht]
		\centering
			\includegraphics[width=0.8\textwidth]{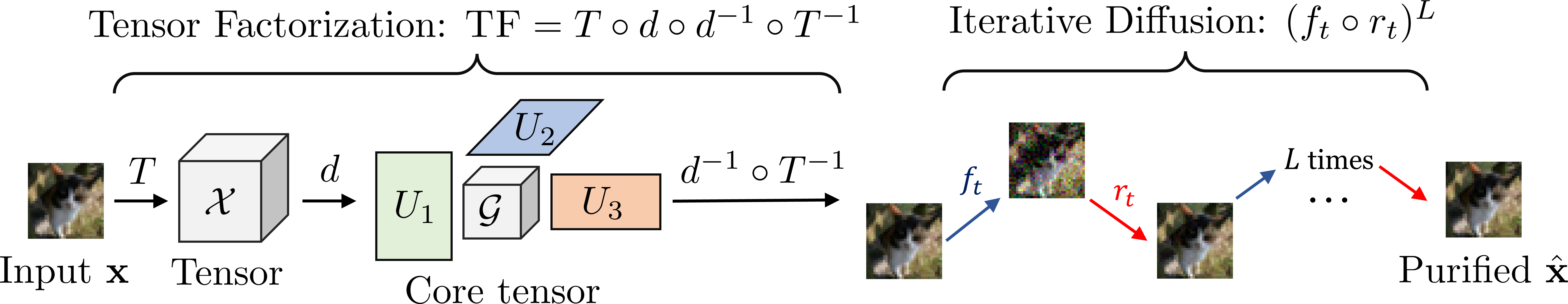}
		\caption{The overall purification process of LoRID: given an input image $\bmx$, LoRID first transforms the image to a tensor and conducts tensor factorization to eliminate some adversarial perturbation. Then, multiple loops of diffusion-denoising, denoted by $f_t$ and $r_t$, at the early stages of the diffusion models are applied to obtain the final purified image $\hat{\bmx}$.}
		\label{lorid_structure}
	\end{figure*}

Our work aims to develop a more robust and efficient diffusion-based purification method to counter emerging adversarial attacks. We first introduce an information-theoretic viewpoint on the diffusion-based purification process, in which the purified signal is considered as the recovered signal from a noisy communication channel. Different from the previous purification~\cite{nie2022DiffPure} centered on the Score-based diffusion~\cite{song2021scorebased}, our work is the first theoretical analysis of the inherent error induced by Markov-based purifications~\cite{blau2022threat, wang2022guided, xiao2022densepure}, which are purifications relying on the Denoising Diffusion Probabilistic Model (DDPM)~\cite{ho2020denoising}. Our theoretical foundation for DDPM (\textbf{Theorem~\ref{theorem:ddpm_converge},~\ref{theorem:ddpm_time} and~\ref{theorem:ddpm_time_upper}}) are essential as they validate the usage of DDPM for purification and leverage its substantial advantage in terms of running time compared to the Score-based (as shown in Table~\ref{table:score_vs_markov}). Our analysis further points out an interesting finding: the purification error (Corollary~\ref{corollary:clean}) can be reduced significantly by conducting multiple iterations at the early time-steps of the DDPM (\textbf{Theorem~\ref{theorem:ddpm_time_loop}}). Particularly, the application of a single purification with a time-step $t$ is theoretically shown to be less beneficial than the looping of $L$ iterations of diffusion-denoising with a time-step of $t/L$ (Fig.~\ref{fig:msse_vs_max_iter}). Our study additionally suggests the usage of \textit{Tucker decomposition}~\cite{5447070}, a higher-order extension of matrix factorization, to attenuate adversarial noise at the high-noise regime (\textbf{Theorem~\ref{theorem:ddpm_tf}}). We realize the advantages of those findings and propose LoRID, a Low-Rank Iterative Diffusion purification method designed to mitigate the purification errors (Fig.~\ref{lorid_structure}). By controlling the purification error, LoRID can effectively increase the diffusion time-step and beat the SOTA robustness benchmark, in both white-box and black-box settings (Table.~\ref{table:score_vs_markov} highlights LoRID's performance in CIFAR-10~\cite{rabanser2017introduction}, and Imagenet~\cite{imagenet_cvpr09}). The main contributions of this work are: 
\begin{itemize}
    \item We establish theoretical bounds on the purification errors of Markov-based purifications. In particular, Theorem~\ref{theorem:ddpm_converge} show that the the adversarial noise will be removed at a distribution-level as the purification time-steps increases. On the other hand, Theorem~\ref{theorem:ddpm_time}, and~\ref{theorem:ddpm_time_upper} point out the purification at the sample-level.    
    \item We show theoretical justifications for looping the early-stages of DDPM (Theorem~\ref{theorem:ddpm_time_loop}), and the usage of Tucker decomposition (Theorem~\ref{theorem:ddpm_tf}) for adversarial purification.
    \item We introduce a Markov-based purification algorithm, called {LoRID} (Alg.~\ref{alg:LoRID}), utilizing early looping and Tucker decomposition and demonstrate rigorously its effectiveness and high performance in three real-world datasets: CIFAR-10/100, CelebA-HQ, and Imagenet.    
\end{itemize}

Our paper is organized as follows. Sect.~\ref{sect:prelim} provides the background and related work of this study. Sect.~\ref{sect:method} consists of our theoretical analysis and the description of our proposed purification LoRID. Sect.~\ref{sect:experiment} provides our experimental results, and Sect.~\ref{sect:conclusion} concludes this paper.

\section{Background and Related Work} \label{sect:prelim}

This section first briefly reviews the Denoising Diffusion Probabilistic Model~\cite{ho2020denoising}, which is the backbone of our diffusion purifications. Then, the related work about the usage of diffusion models as adversarial purifiers is discussed. Finally, we briefly discuss the Tucker decomposition, which is a component utilized by our method.

\textbf{Denoising Diffusion Probabilistic Models (DDPMs)} are a class of generative models that, during training, iteratively adding noise to input signals, then learning to denoise from the resulting noisy signal. Formally, given a data point $\bmx_0$ sampled from the data distribution $ q(\bmx_0)$, a \textit{forward diffusion process} from clean data $\bmx_0$ to $\bmx_T$ is a Markov-chain that gradually adds Gaussian noise, denoted by $\mathcal{N}$, to the data according to a variance schedule $\{\beta_t\in (0,1)\}^T_{t=1}$: $( \bmx_{1:T}|\bmx_0) := \prod^{T}_{t=1}q(\bmx_t|\bmx_{t-1})$, where
\begin{align}
   &q(\bmx_{t}|\bmx_{t-1}) := \mathcal{N}(\bmx_t;\sqrt{1-\beta_t}\bmx_{t-1},\beta_t \mathbf{I}) \label{eq:forward}
\end{align}
The objective of DDPM is to learn the joint distribution $p_{\theta}(\bmx_{0:T})$, called the reverse process, which is defined as another Markov-chain with learned Gaussian transitions  $p_{\theta}(\bmx_{0:T}) = p(\bmx_T) \prod_{t=1}^{T} p_{\theta}(\bmx_{t-1} | \bmx_t)$, where
\begin{align}
 p_{\theta}(\bmx_{t-1} | \bmx_t) &:= \mathcal{N}(\bmx_{t-1}; \boldsymbol{\mu}_{\theta}(\bmx_t, t), \mathbf{\Sigma}_{\theta}(\bmx_t, t))
\end{align}
starting with $p(\bmx_T) = \mathcal{N}(\bmx_T; \mathbf{0}, \mathbf{I})$. The mean $\boldsymbol{\mu}_{\theta}(\bmx_t, t)$ is a neural network parameterized by $\theta$, and the variance $ \mathbf{\Sigma}_{\theta}(\bmx_t, t)$ can be either time-step dependent constants~\cite{ho2020denoising} or learned by a neural network~\cite{nichol2021improved}. A notable property of the forward process is that it admits sampling $\bmx_t$ at an arbitrary time-step $t$ in closed form: using the notation $\alpha_t := 1 - \beta_t$ and $\bar{\alpha}_t := \prod_{s=1}^{t} \alpha_s$, we have
\[
q(\bmx_t | \bmx_0) = \mathcal{N}\left(\bmx_t; \sqrt{\bar{\alpha}_t} \bmx_0, (1 - \bar{\alpha}_t) \mathbf{I}\right)
\]
Using the reparameterize trick, we can define the forward diffusion process to the time-step $t$ as $f_t$:
\begin{align}
    \bmx_t = f_{t}(\bm x_0) := \sqrt{\bar{\alpha}_t} \bmx_0 + \sqrt{1 - \bar{\alpha}_t}\boldsymbol{\epsilon}_0 \label{eq:diff_xt}
\end{align}
where $ \boldsymbol{\epsilon}_0$ is  a standard Gaussian noise.

For the reverse process, the recovered signal from the time-step $t$ can be written as~\cite{ho2020denoising}:
\begin{align}
    \Tilde{\bmx}_0(t) = \frac{1}{\sqrt{\bar{\alpha}_t}} \bmx_t - \frac{\sqrt{1 - \bar{\alpha}_t}}{ \sqrt{\bar{\alpha}_t}}\boldsymbol{\epsilon}_{\theta}(\bmx_t, t)
\end{align}
where $\boldsymbol{\epsilon}_{\theta}$ is a function approximator predicting $\boldsymbol{\epsilon}$ from $\bmx_t$, i.e., the \textit{noise matching term}. Given that, we have
\begin{align}
    \Tilde{\bmx}_0(t)  - \bmx_0 &=  \frac{\sqrt{1 - \bar{\alpha}_t}}{ \sqrt{\bar{\alpha}_t}} \left( \boldsymbol{\epsilon}_0 - \boldsymbol{\epsilon}_{\theta}\left( \bmx_t, t
    \right)\right) 
\end{align}
Thus, the approximator $\boldsymbol{\epsilon}_{\theta}$ can be trained using MSE loss:
\begin{align}
    L(\theta) := 
    \mathbb{E}_{t, \bmx_0, \boldsymbol{\epsilon}} \left[ 
    \left \| 
    \boldsymbol{\epsilon} - \boldsymbol{\epsilon}_\theta
    \left(
    \bmx_t, t
    \right)
    \right \|^2
    \right] \label{eq:training_ddpm}
\end{align}

\textbf{Diffusion models as adversarial purifiers.} Diffusion-based purification schemes can be categorized into Markov-based purification (or DDPM-based), and Score-based purification, which utilize DDPM~\cite{ho2020denoising} and Score-based diffusion model~\cite{song2021scorebased} to purify the adversarial examples, respectively. In this work, we focus on the Markov-based methods, which typically diffuse the adversarial input $\bmx_a$ to a certain time-step $t$, then utilize the DDPMs to iterativly solve the reverse process as given in~\cite{ho2020denoising}:
\begin{small}
\begin{align}
    \hat{{\bmx}}_{t-1} = \frac{1}{\sqrt{{\alpha}_t}} \left( \hat{\bmx}_{t} - \frac{1-\alpha_t}{\sqrt{1 - \bar{\alpha}_t}} \boldsymbol{\epsilon}_{\theta}(\bmx_{t},t) \right) + 
   \frac{\beta_t (1-\bar{\alpha}_{t-1})}{1-\bar{\alpha}_{t}}  \boldsymbol{\epsilon} \label{eq:reverse_iterative}
\end{align}
\end{small}We denote the process of running (\ref{eq:reverse_iterative}) iteratively, starting from $\hat{\bmx}_t := \bmx_t$ to finally obtain $ \hat{{\bmx}}_{0}$ by $ r_{t}$ and write $ \hat{{\bmx}}_{0} = r_{t}(\bmx_t)$. The whole process of purification is, therefore, can be referred by the composition $ r_t \circ f_t$.

Recent Markov-based purifications often apply a modified version of (\ref{eq:reverse_iterative}). The work~\cite{blau2022threat} uses a re-scaled version of (\ref{eq:reverse_iterative}) with a larger noise term for purification. The Guided-DDPM purification~\cite{wang2022guided} introduces a guided term encouraging the purified image to be close to the adversarial image in the reverse process to protect the sample's semantics. On the other hand, DensePure~\cite{xiao2022densepure} computes multiple reversed samples using (Eq.~\ref{eq:reverse_iterative}) and determines final predictions by majority voting. We find that the recent findings of~\citet{lee2023robust} are the most closely related to this work: they observe that the \textit{gradual noise-scheduling strategy}, which involves looping several times during the early stages of the DDPM, can enhance defense mechanisms. However, there is no theoretical justification for this strategy. Furthermore, as the attackers in their threat models are unaware of this defense, it is unclear whether gradual noise-scheduling truly offer better robustness against practical white-box attackers. 




Despite the differences, all methods emphasize a too-large $t$ would damage the global label semantics from the purified sample. While the theoretical statement for this is stated in the case of Score-based purifications~\cite{nie2022DiffPure}, the counterpart for Markov-based is lacking. One of our contributions is the theoretical statement for the Markov-based in Theorem~\ref{theorem:ddpm_converge} and Theorem~\ref{theorem:ddpm_time}. Another aspect that distinguishes the technicality our method, LoRID, from previous work is the use of a large number of loops (typically between $10$ and $40$) in the early stages (about $1\%-10\%$ of the total time-step) of the DDPM, along with the implementation of Tucker decomposition. 

\textbf{Tucker Decomposition}, also known as higher-order singular value decomposition (HOSVD)~\cite{5447070}, is a mathematical technique used in multilinear algebra and data analysis, and can be viewed as an extension of the concept of singular value decomposition (SVD) for higher-dimensional data arrays or tensors~\cite{kolda2009tensor}.  Computing the Tucker decomposition of a tensor can encode the essential information and structure of the tensor into a set of core tensor and factor matrices. It's widely used in various fields such as signal processing, image processing, neuroscience, data compression, and in the line of the proposed work, feature extraction from high-dimensional data~\cite{kolda2009tensor}.

In our purification context, as the latent tensor $\mathcal{X}$ is obtained from an original image $\bmx$ via a tensorization process~\cite{bhattarai2023robust}, denoted by $ \mathcal{X}:= T(\bmx)$, the overall tensor-factorization denoising process (Fig.~\ref{lorid_structure}) can be referred by the following:
\begin{align}
    \hat{\bmx} = \textsc{TF}(\bmx + \boldsymbol{\epsilon} )
\end{align}
where $ T^{-1}$ denotes the recover of the image from the latent space, and $\textsc{TF}$ is defined as $ \textsc{TF}:= T^{-1}\circ d^{-1} \circ d \circ T$. The details of this operation is provided in Appx.~\ref{appx:tucker}.

As both Tucker and tensorization are linear transformations, the denoising-reconstruction error can be bounded as:
\begin{align}
    \| \bmx - \textsc{TF}( \bmx + \boldsymbol{\epsilon} ) \| \leq&  \| \bmx - \textsc{TF}( \bmx) \| + \| \textsc{TF}(\boldsymbol{\epsilon} )\| . \label{eq:error_bound_tf}
\end{align}
Here, the first term, denoted as $\textsc{E}_{\textsc{Tucker}}:=\| \bmx - \textsc{TF}( \bmx) \|$ represents the error introduced by not capturing the full variance in each mode of the data through the Tucker decomposition. The second term $\|\textsc{TF}(\boldsymbol{\epsilon} )\| $
represents the error caused by the original noise on $\bmx $ remained after the denoising process. $\textsc{E}_{\textsc{Tucker}}$ can be bounded further by (\cite{DeLathauwer} Property 10; \cite{Hackbusch2012TensorSA} Theorem 10.2): $
    \textsc{E}_{\textsc{Tucker}} \leq \sum_{n=1}^{N} \sum_{i_n=r_n+1}^{I_n} \left(\sigma_{i_n}^{(n)} \right)^2$,
where $\{ \sigma_{i_n}^{(n)} \}_{i_n =  1}^{I_n}  $ is the singular values of the mode-$n$ unfolding of the tensor $\mathcal{X}$.

\section{Method} \label{sect:method}

This section provides theoretical results on different aspects of Markov-based purification (\ref{eq:reverse_iterative}) and the details for our proposed adversarial purification algorithm LoRID.
\begin{itemize}
    \item Subsect.~\ref{subsect:diffusion_purification} provides Theorem~\ref{theorem:ddpm_converge} about the theoretical removal of the adversarial noise as the diffusion time-step $t$ in the Markov-based diffusion model increases at the distribution-level. It is the counterpart of Theorem 3.1 in~\cite{nie2022DiffPure} for Score-based purification.
    \item The purification error between the clean and the purified images at the sample-level are further characterized in Theorem~\ref{theorem:ddpm_time} and~\ref{theorem:ddpm_time_upper} in Subsect.~\ref{subsect:diffusion_purification}. While Theorem~\ref{theorem:ddpm_time_upper} can be viewed as an adaptation of Theorem 3.2 from Score-based to Markov-based purification, to the best of our knowledge, the lower bound on the reconstruction error in Theorem~\ref{theorem:ddpm_time} has not been previously established for any diffusion-based purification methods.
    \item Subsect.~\ref{subsect:purification_error} demonstrates how we realize our theoretical analysis into practical measurement. Particularly, We analyze the intrinsic purification error arising from the Markov-based purification process (Corollary~\ref{corollary:clean}) and identify the advantage of looping the early time-steps of the diffusion models for the purification task (Theorem~\ref{theorem:ddpm_time_loop}). The result suggests that, with the same effective diffusion-denoising steps, looping can reduce the intrinsic purification error significantly.
    \item Subsect.~\ref{subsect:purification_error} also studies and validates the usage of Tucker Decomposition combined with Markov-based purification at the high-noise regime (Theorem~\ref{theorem:ddpm_tf}).
    \item Based on the theoretical analysis, we design LoRID, the Low-Rank Iterative Diffusion method to purify adversarial noise. Its description is provided in Subsect.~\ref{subsect:methodology}.
\end{itemize}


\subsection{Markov-based Purification} \label{subsect:diffusion_purification}

Intuitively, the diffusion time-step $t$ need be large enough to remove adversarial perturbations; however, the image's semantics will also be removed as $t$ increases. That observation is captured in the following Theorem~\ref{theorem:ddpm_converge}, which states that the KL-divergence between the distributions of the clean images and the adversarial images  converges as $t$ increases:
\begin{theorem} \label{theorem:ddpm_converge}
Let $\left\{\bmx^{(i)}_t \right\}_{t \in \{0,...,T \}},  i \in \{1,2\}$ be two diffusion processes given by the forward equation (\ref{eq:forward}) of a DDPM. Denote $q^{(1)}_t$ and $q^{(2)}_t$ the distributions of $\bmx_t^{(1)}$ and $\bmx_t^{(2)}$, respectively. Then, for all $t \in \{0,..., T-1\}$, we have
\begin{align*}
    D_{KL}\left(q^{(1)}_t||q^{(2)}_t \right) \geq D_{KL}\left(q^{(1)}_{t+1}||q^{(2)}_{t+1} \right)
\end{align*}
\end{theorem}
\textit{Sketch of proof (proof in Appx~\ref{appx:theorem:ddpm_converge}).} While Theorem~\ref{theorem:ddpm_converge} resembles that stated for the Score-based purification~\cite{nie2022DiffPure}, its proof is greatly different since the DDPM's diffusion is not controlled by an Stochastic Differential Equation. Instead, we leverage the underlying Markov process governing the forward diffusion of DDPM (\ref{eq:forward}), and show $D_{KL}\left( q^{(1)}_{t+1} ||q^{(2)}_{t+1} \right) + D_{KL}\left( q^{(1)}(\bmx_{t}| \bmx_{t+1}) ||q^{(2)}(\bmx_{t} | \bmx_{t+1}) \right)  = D_{KL}\left( q^{(1)}_{t} ||q^{(2)}_{t} \right)  $ by expanding the KL-divergence between $q^{(1)}(\bmx_{t+1}, \bmx_t)$ and $q^{(2)}(\bmx_{t+1}, \bmx_t)$. Then, due to the non-negativity of the KL-divergence, we have the Theorem.


Note that Theorem~\ref{theorem:ddpm_converge} captures the purification at the distribution level. Similar to the Score-based purification~\cite{nie2022DiffPure}, we are also interested in the purification of the DDPM at the instance level. In fact,
the variational bound (Eq.~\ref{eq:training_ddpm}) suggests that the reconstruction error $\| \hat{\bmx}_0(t)  - \bmx_0  \|$ is directly proportional to the DDPM's training objective. However, that objective, $L(\theta)$,
is for all time-steps, while the purification error only depends on the one time-step, at which, the reverse process is applied to recover $\hat{\bmx}_0(t)$. Intuitively, as $t$ increases, the argument of the approximator $\sqrt{\bar{\alpha}_t} \bmx_0 + \sqrt{1 - \bar{\alpha}_t}\boldsymbol{\epsilon}$ contains less information about the noise $\boldsymbol{\epsilon} $, thus, results in a higher error. The following two Theorems formalize that intuition:

\begin{theorem} \label{theorem:ddpm_time}
Let $\left\{\bmx_t \right\}_{t \in \{0,...,T \}}$ be a diffusion process defined by the forward equation (\ref{eq:forward}) where $\bmx_0$ is the adversarial sample. i.e, $\bmx_0 = \bmx_{clean} + \boldsymbol{\epsilon}_a$. For any time $t$, we have
\begin{align}
    \mathbb{E} \left [\|\hat{\bmx}_0(t) - \bmx_{clean} \| \right]\geq  \textup{MMSE}\left( \frac{{\bar{\alpha}_t}}{{1- \bar{\alpha}_t}}  \right) - \| \boldsymbol{\epsilon}_a \|  \label{eq:error_bound}
\end{align}
where the expectation is taken over the distribution of $\bmx_{clean}$ and $\textup{MMSE}(\textsc{SNR})$ is the minimum mean-square error achievable by optimal estimation of the input given the output of Gaussian channel with a signal-to-noise ratio of $\textsc{SNR}$. The function $\textup{MMSE}(\textsc{SNR}) $ has the following form~\cite{mmse_mi}:
\begin{align}
 1 - \frac{1}{\sqrt{2\pi}} \int_{-\infty}^{\infty} e^{-y^2/2 }\textup{tanh} \left( \textsc{SNR} - \sqrt{\textsc{SNR}} y
\right) \textup{d} y. \label{eq:mmse}
\end{align}
\end{theorem}

\begin{theorem} \label{theorem:ddpm_time_upper}
Additionally to the conditions stated in Theorem~\ref{theorem:ddpm_time}, if the DDPM is able to recover the original signal $\bmx_0 $ within an error $\boldsymbol{\delta}_{\textsc{DDPM}}(t)$ in the expectation, i.e., for all $t$,
\begin{align}
     \mathbb{E} \left[ 
    \left \| 
   \hat{\bmx}_0(t)  - \bmx_0
    \right \|
    \right] \leq \mathbb{E} \left[ 
    \left \| 
   \hat{\bmx}^*_0( \mathbf{y}_t)  - \bmx_0
    \right \|
    \right]  + \boldsymbol{\delta}_{\textsc{DDPM}}(t) \label{eq:assume_delta_ddpm}
\end{align}
 where $\hat{\bmx}^*_0( \mathbf{y}_t) $ is  the best estimator of $\bmx_0$ given $\mathbf{y}_t  = ({\sqrt{\bar{\alpha}_t}}/{\sqrt{1- \bar{\alpha}_t}}) \bmx_0 + \boldsymbol{\epsilon}_0$, then, we have the reconstructed error $\mathbb{E} \left [\|\hat{\bmx}_0(t) - \bmx_{clean} \| \right]$ is upper-bounded by:
\begin{align}
     \textup{MMSE}\left( \frac{{\bar{\alpha}_t}}{{1- \bar{\alpha}_t}}  \right) + \boldsymbol{\delta}_{\textsc{DDPM}}(t) + \| \boldsymbol{\epsilon}_a \|  \label{eq:error_upper_bound}
\end{align}
\end{theorem}

\textit{Sketch of proofs (proofs in Appx~\ref{appx:theorem:ddpm_time} and~\ref{appx:theorem:ddpm_time_upper}).}  The proofs of both theorems consider the forwarding diffusion of the DDPM as a Gaussian channel, and the purification task is equivalent to the reconstruction of the channel's input. Given a purification time-step $t$, i.e., the time-step we decide to start the denoising/purification process, the equivalent Gaussian channel would have an effective signal-to-noise ($\textsc{SNR}$) ratio of ${{\bar{\alpha}_t}}/({{1- \bar{\alpha}_t}})$. Intuitively, the higher the time-step, the smaller the value of ${{\bar{\alpha}_t}}/({{1- \bar{\alpha}_t}})$, and, even with an optimal denoiser, the more inherent error are introduced to the purified sample. In fact, the expression $\textup{MMSE}({{\bar{\alpha}_t}}/({{1- \bar{\alpha}_t}}))$ appearing in both theorems capture that intrinsic error. Unfortunately, there is currently no closed-form for that expression. We follow previous work studying noisy Gaussian channel~\cite{mmse_mi} and provide its integral form in expression (\ref{eq:mmse}).

\textbf{Remark.}\textit{ Regarding $\boldsymbol{\delta}_{\textsc{DDPM}}(t)$, it captures how well the trained-DDPM can recover the input given its noisy signal at time-step $t$. The assumption that $\boldsymbol{\delta}_{\textsc{DDPM}}(t)$ bounds the reconstruction error (\ref{eq:assume_delta_ddpm}) is a weaker version of the assumption made by~\cite{song2021scorebased} in the analysis of the Score-based diffusion, which is also utilized to upper-bound the error induced by Score-based purification~\cite{ho2020denoising}. In fact, both works assume the Score-based diffusion model can \textbf{perfectly} learn the score function $\nabla_x \log p(\bmx_0)$ to establish their theoretical results. }

\begin{figure}[ht]
		\centering
\includegraphics[width=0.995\linewidth]{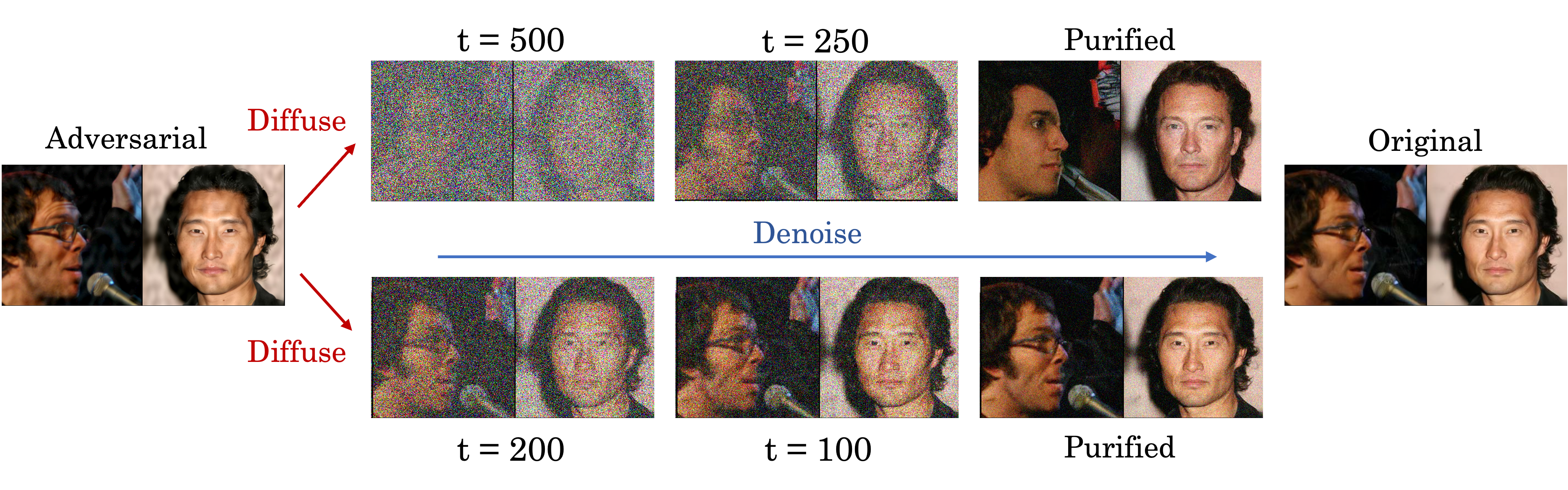}
		\caption{Illustration of adversarial purification using DDPM. The adversarial samples (left) is purified from the time-step $t=200$ (bottom) and $t=500$ (top) to recover the original samples (right). The middles  show $\hat{{\bmx}}_{t}$ (Equation (\ref{eq:reverse_iterative})) obtained by iteratively denoising to the indicated intermediate time-steps. The top purification with a too large time-step induces unavoidable error (Theorem~\ref{theorem:ddpm_time}).}
		\label{fig:illu_purified}
\end{figure}

To conclude this subsection, we illustrate the impact of the time-step $t$ on the purification process based on DDPM in Fig.~\ref{fig:illu_purified}. When $t$ is chosen appropriately, all the terms on the right-hand-side of (\ref{eq:error_upper_bound}) of Theorem~\ref{theorem:ddpm_time_upper} are controlled, which enforces a small difference between the clean input and and recovered signal. This means not only the adversarial noises are removed but also the purified images maintain the semantic of the original data. This is reflected in the purified images at the bottom of Fig.~\ref{fig:illu_purified})
 However, when $t$ is too large as illustrated at the top of Fig.~\ref{fig:illu_purified}, the diffusion-based purification induces an intrinsic error reflected in the term $\textup{MMSE}\left( {{\bar{\alpha}_t}}/{{(1- \bar{\alpha}_t} )}  \right) $ of Theorem~\ref{theorem:ddpm_time}. This error makes the purified images inevitably different from the original signal.


\subsection{Controlling Purification Error} \label{subsect:purification_error}

This subsection studies the inherent error introduced by the purification process and demonstrates why it instigates a better purification scheme based on looping the early stage of the DDPM and the utilization of Tucker decomposition. 

Our analysis starts with the consideration of the trivial case in which there is no adversarial noise. By combining the two Theorems~\ref{theorem:ddpm_time} and~\ref{theorem:ddpm_time_upper}, we have the following corollary:
\begin{corollary} \label{corollary:clean}
Given the assumptions in Theorem~\ref{theorem:ddpm_time_upper}, the intrinsic purification error on a clean purification input $\bmx_0 = \bmx_{clean}$, i.e., $ \boldsymbol{\epsilon}_a = 0$, is bounded by
\begin{align}
      &\textup{MMSE}\left( \frac{{\bar{\alpha}_t}}{{1- \bar{\alpha}_t}}  \right) \leq \mathbb{E} \left [\|\hat{\bmx}_0(t) - \bmx_{clean} \| \right]  \nonumber \\
      \leq &  \textup{MMSE}\left( \frac{{\bar{\alpha}_t}}{{1- \bar{\alpha}_t}}  \right) + \boldsymbol{\delta}_{\textsc{DDPM}}(t)  \label{eq:error_clean}
\end{align}
\end{corollary}

The corollary reflects the strong connection between the purification error and the MMSE term. Especially, when the DDPM is well-trained, the gap $\boldsymbol{\delta}_{\textsc{DDPM}}(t)$ between the lower and upper bounds becomes small, and $\mathbb{E} \left [\|\hat{\bmx}_0(t) - \bmx_{clean} \| \right]$ becomes more similar to $\textup{MMSE}\left( {{\bar{\alpha}_t}}/{{(1- \bar{\alpha}_t)
}}  \right) $. This observation motivates us to investigate purification schemes that minimize the MMSE. 

\textbf{Looping at early time-steps.} Several recent work observed that repetitive usage of the diffusion-denoising steps in parallel~\cite{wang2022guided, nie2022DiffPure} or sequential \cite{lee2023robust} can enhance system robustness against adversarial attacks. However, too many diffusion-denoising calls would not only diminish robustness gain but also degrade the clean accuracy significantly. Hence, we tackle the following question: \textit{Given a fixed number of denoiser's call, i.e., total number of diffusion-denoising steps, for the sake of adversarial purification, should we diffusion-denoising multiple loops of the DDPM at the earlier time-steps or utilize a few loops with large time-steps?} 

We now provide theoretical justification for the usage of multiple  loops in purification. Specifically, we want to compare the purification to the time-step $t$, i.e.,  denoted by $r_t \circ f_t$, and the purification of $L$ loops to the time-step $t/L$,  i.e., $(r_{t/L} \circ f_{t/L} )^{L}$. By denoting the output of $l$  times DDPM-purification to time-step $t$, $ \hat{\bmx}^{l}_0(t) :=  (r_{t/L} \circ f_{t/L} )^{l} (\bmx_0)$, we formalize the impact of looping purification via the following Theorem~\ref{theorem:ddpm_time_loop}:


\begin{theorem} \label{theorem:ddpm_time_loop}
Let $\left\{\bmx_t \right\}_{t =0}^T$ be a diffusion process defined by the forward (\ref{eq:forward}) where $\bmx_0$ is the adversarial sample. i.e, $\bmx_0 = \bmx_{clean} + \boldsymbol{\epsilon}_a$. For any given time $t$, we have the reconstructed error $ \mathbb{E} \left [\|\hat{\bmx}^L_0(t/L) - \bmx_{clean} \| \right] $ is upper-bouned by:
\begin{small}
\begin{align}
      L \times \left(  \textup{MMSE}\left( \frac{{\bar{\alpha}_{t/L}}}{{1- \bar{\alpha}_{t/L}}}  \right)  + \boldsymbol{\delta}_{\textsc{DDPM}}\left( \frac{t}{L}\right)  \right)  + \| \boldsymbol{\epsilon}_a \|  \label{eq:error_bound_loop}
\end{align}
\end{small}
where the expectation is taken over the distribution of $\bmx_{clean}$ (Proof in Appx.~\ref{appx:theorem:ddpm_time_loop}).
\end{theorem}


Note that the upper-bound on the reconstruction error of $(r_{t/L} \circ f_{t/L} )^{L}$ is controlled by $L \times \textup{MMSE}\left( {{\bar{\alpha}_{t/L}}} / ({{1- \bar{\alpha}_{t/L}}} ) \right) $, instead of $ \textup{MMSE}\left( {{\bar{\alpha}_{t}}} / ({{1- \bar{\alpha}_{t}}} ) \right) $ as in the vanilla purification scheme $r_t \circ f_t$. For an illustration of the impact of looping the diffusion-denoising, we consider the input to compute the MMSE as standard Gaussian. The MMSE is then given by $ \textup{MMSE}(\textsc{SNR}) = 1 /(1 + \textsc{SNR})$ (instead of the integral form (\ref{eq:mmse})). We further take the values of $\bar{\alpha}_t$ in DDPM~\cite{ho2020denoising} and plot $L \times \textup{MMSE}\left( {{\bar{\alpha}_{t/L}}} / ({{1- \bar{\alpha}_{t/L}}} ) \right) $ as a function of $L$ in Fig.~\ref{fig:msse_vs_max_iter}. The result shows that purification at a small time-step with a large number of iteration is greatly beneficial for the purification error.

\textbf{Tucker Decomposition for High-noise Regime.} We now study the utilization of DDPM and Tucker Decomposition to purify the adversarial samples, which is characterized by the operations $r_{t}  \circ f_{t}$ and $\textsc{TF}= T^{-1}\circ d^{-1} \circ d \circ T$. From the previous analysis, the reconstruction error induced by the two methods are bounded by:
\begin{small}
\begin{align}
    \textsc{MSE}_{r_{t}  \circ f_{t}}(\boldsymbol{\epsilon}_a) &\leq \textup{MMSE}\left( \frac{{\bar{\alpha}_t}}{{1- \bar{\alpha}_t}}  \right) + \boldsymbol{\delta}_{\textsc{DDPM}}(t) + \| \boldsymbol{\epsilon}_a \| \label{eq:error_bound_ddpm_recall} \\
     \textsc{MSE}_{\textsc{TF}}(\boldsymbol{\epsilon}_a) &\leq  \textsc{E}_{\textsc{Tucker}} + \left \| \textsc{TF}(\boldsymbol{\epsilon}_a ) \right \|  \label{eq:error_bound_tf_recall}
\end{align}
\end{small}where (\ref{eq:error_bound_ddpm_recall}) is from Theorem~\ref{theorem:ddpm_time_upper} and (\ref{eq:error_bound_tf_recall}) is from (\ref{eq:error_bound_tf}). Here, $ \textsc{MSE}_{r_{t}  \circ f_{t}}(\boldsymbol{\epsilon}_a)$ and $\textsc{MSE}_{\textsc{TF}}(\boldsymbol{\epsilon}_a)$ denote the reconstruction error of the $r_{t}  \circ f_{t}$ and  $\textsc{TF}$ purification schemes (stated in (\ref{eq:error_upper_bound}) and (\ref{eq:error_bound_tf}), respectively). We now provide the upper-bounds of an integration of Tucker Decomposition into DDPM purification in the following Theorem~\ref{theorem:ddpm_tf}.

\begin{theorem} \label{theorem:ddpm_tf}
    The reconstruction errors introduced of the purification $ r_{t}  \circ f_{t} \circ \textsc{TF}$ is bounded by:
    \begin{align}
    \textsc{MSE}_{r_{t}  \circ f_{t} \circ \textsc{TF} }(\boldsymbol{\epsilon}_a) & \leq \textup{MMSE}\left( \frac{{\bar{\alpha}_t}}{{1- \bar{\alpha}_t}}  \right) + \boldsymbol{\delta}_{\textsc{DDPM}}(t) \nonumber \\ &+ \textsc{E}_{\textsc{Tucker}} + \left \| \textsc{TF}(\boldsymbol{\epsilon}_a ) \right \| \label{eq:error_bound_tfddpm}
\end{align}
(Proof in Appx.~\ref{appx:theorem:ddpm_tf})
\end{theorem}
Intuitively, comparing to the purification $ r_{t}  \circ f_{t}$, this purification process $ r_{t}  \circ f_{t} \circ \textsc{TF} $ have a better upper bound when the Tucker Decomposition can reduce the adversarial noise before forwarding the signal to the DDPM, i.e., when $\textsc{E}_{\textsc{Tucker}} + \left \| \textsc{TF}(\boldsymbol{\epsilon}_a ) \right \|  < \left \| \boldsymbol{\epsilon}_a  \right \|$, which suggests the usage of Tucker Decomposition at a high-adversarial-noise regime.

\subsection{LoRID: Low-Rank Iterative Diffusion for Adversarial Purification} \label{subsect:methodology}

Based on the above analysis, we propose LoRID, Low-Rank Iterative Diffusion algorithm for adversarial furification. Generally, LoRID consists of two major steps: Tensor factorization, and diffision-denoising. So far, our manuscript has considered four different configurations of LoRID, depending on the usage of looping and on how the TF and diffusions are coupled: Tensor-factorization $\textsc{TF}$, diffusion-denoising $r_t \circ f_t$, looping $(r_{t/L} \circ f_{t/L} )^{L}$, and Tensor-factorization with diffusion-denoising $\textsc{TF} \circ r_t \circ f_t$. However, the default configuration that we refer to with LoRID would utilize both Tucker Decomposition (step 1) and multiple loops of diffusion-denoising (step 2), which can be described by the expression $\textsc{TF} \circ (r_{t/L} \circ f_{t/L} )^{L} $. The pseudo-code of LoRID is described in Appendix.~\ref{appx:algorithm}.

\section{Experiments} \label{sect:experiment}
This section is about our experimental setting and robustness results: Subsect.~\ref{subsect:exp_set} highlights the experimental settings and Subsect.~\ref{sect:exp:robust_results} reports our experimental results.



\subsection{Experimental Setting}  \label{subsect:exp_set}

\paragraph{Datasets and attacked architectures.} We evaluate LoRID on CIFAR-10/100~\cite{rabanser2017introduction}, CelebA-HQ~\cite{karras2018progressivegrowinggansimproved}, and ImageNet~\cite{imagenet_cvpr09}. Comparisons are made against SOTA defense methods reported by RobustBench~\cite{croce2021robustbench} on CIFAR-10 and ImageNet, and against DiffPure~\cite{nie2022DiffPure}, a score-based diffusion purifier, on CIFAR-10, ImageNet, and CelebA-HQ. We use the standard WideResNet~\cite{zagoruyko2017wide} architecture for classification, evaluating defenses using standard accuracy (pre-perturbation) and robust accuracy (post-perturbation). When the gradients is not needed (black-box setting) in CIFAR-10, all methods are evaluated 10000 test images. On the other hand, due to the high computational cost of computing gradients for adaptive attacks against diffusion-based defenses, we assess the methods on a fixed subset of 512 randomly sampled test images, consistent with previous studies~\cite{nie2022DiffPure, lee2023robust}. Further experimental details are provided in Appx.~\ref{append:Implmentation} with EOT=20.

\paragraph{Attacker settings.} We consider two common threat models: black-box and white-box. In both scenarios, the attacker has full knowledge of the classifier. However, only in the white-box setting, the attacker also knows about the purification scheme.~\footnote{In our white-box setting, the attacker is aware of both $t$ and $L$ in our LoRID framework and can fully backpropagate through the DDPM, making this scenario even stronger than the white-box assumption used by~\citet{lee2023robust}.
} For black-box, we adapt~\cite{nie2022DiffPure, lee2023robust} and evaluate defense methods against AutoAttack~\cite{croce2020reliable} in CIFAR-10/100 and BPDA+EOT~\cite{10.1145/3524619} in CelebA-HQ. For white-box, we also follow the literature and consider AutoAttack and PGD+EOT~\cite{DBLP:journals/corr/abs-1907-00895}. 

However, white-box attacks require gradient backpropagation through the diffusion-denoising path, causing memory usage to increase linearly with diffusion step $t$. This makes exact gradient attacks infeasible on larger datasets like CelebA-HQ and ImageNet~\cite{kang2024diffattack}. Therefore, all existing work rely on some approximations of the gradients to conduct white-box attacks on those dataset~\cite{nie2022DiffPure, lee2023robust}.\footnote{While the \textit{adjoint}~\cite{nie2022DiffPure} against the Score-based purification is claimed to be exact, it relies on underlying numerical solvers and they can introduce significant error. We observe that using adjoint-gradients results in significantly weaker attack than using surrogate, which is also observed and reported by~\citet{kang2024diffattack, lee2023robust}.} To the best of our knowledge, The strongest approximation to date is the \textit{surrogate} method~\cite{lee2023robust}, which denoises noisy signals using fewer denoising steps~\cite{song2020denoising}. This approach reduces the number of denoiser calls while effectively simulating the original process (details in Appendix~\ref{appx:attacking_surrogate}). In summary, we use exact gradients for CIFAR-10 and the surrogate method for CelebA-HQ and ImageNet in our white-box attacks.

\paragraph{LoRID settings.} LoRID requires the specification of both the time-step $t$ and the looping number $L$, which are crucial for its iterative process. These hyperparameters are generally selected by evaluating the classifier's performance on the clean dataset, with $t$ and $L$ chosen to maintain acceptable clean accuracy. Further details on this parameter selection process are provided in Appx.~\ref{appx:calibration}. We report those parameters as a tuple $(t,L)$ next to the name of our method. Additionally, obtaining an accurate Tucker decomposition for large datasets can be computationally intensive. Therefore, in such cases, LoRID is applied solely with Markov-based purification. In our results, the use of Tucker decomposition is denoted by $\textsc{TF}$ next to the method's name, e.g. $(\textsc{TF}, t,L)$.

\subsection{Robustness Results} \label{sect:exp:robust_results}
We compare LoRID with the SOTA adversarial training methods documented by RobustBench~\cite{croce2021robustbench}, as well as leading adversarial purification techniques, against strong $L_\infty$ and $L_2$ attacks.


\begin{table}[!ht] 
		\centering
		\caption{Standard accuracy and robust accuracy against AutoAttack $L_\infty$ ($\epsilon={8}/{255}$) on CIFAR-10. * indicates the usage of extra data. The gray and white boxes indicate the black-box and white-box attacks.}
		\label{tab:L_inf_cifar10}
		\footnotesize\addtolength{\tabcolsep}{0pt}
		\vspace{-4pt}
		\begin{tabular}{ccc}
			\hline
			\textbf{Method} &  {Standard Acc} & {Robust Acc} \\
			\hline & \\[-2.4ex]

		      \rowcolor{lime}
        \multicolumn{3}{c}
			{WideResNet-28-10} 
			\\
            \hline
             \rowcolor{lightgray}
			\text{ \cite{zhang2020geometry}}*  & 89.36 & 59.96
			\\
             \rowcolor{lightgray}
			\text{ \cite{wu2020adversarial}}*  & 88.25 & 62.11
			\\
             \rowcolor{lightgray}
			\text{ \cite{gowal2020uncovering}}* & 89.48 & 62.70
			\\
             \rowcolor{lightgray}
			\text{ \cite{wu2020adversarial}}  & 85.36 & 59.18
			\\
             \rowcolor{lightgray}
			\text{ \cite{rebuffi2021fixing}}  & 87.33 & 61.72
			\\
            \rowcolor{lightgray}
			\text{ \cite{gowal2021improving}}  & 87.50 & 65.24 \\ 
            \rowcolor{lightgray}
            \textbf{LoRID} ($39,5$) & \textbf{90.41} & \textbf{88.39} \\
            \hline
            \cite{wang2022guided} & 85.66 & 33.48 \\
                \cite{nie2022DiffPure}  &  \textbf{89.02} &  46.88
    			\\
                \textbf{LoRID} ($20,24$) & 84.20 & \textbf{59.14 }
   \\
			\hline & \\[-2.4ex]
			\rowcolor{lime}
   \multicolumn{3}{c}
			{WideResNet-70-16}  \\
			\hline
        \rowcolor{lightgray}
   \text{ \cite{gowal2020uncovering}}* & 91.10 & 66.02
			\\
            \rowcolor{lightgray}
			\text{ \cite{rebuffi2021fixing}}*  & \textbf{92.23} & 68.56
			\\
            \rowcolor{lightgray}
			\text{ \cite{gowal2020uncovering}} & 85.29 & 59.57
			\\
            \rowcolor{lightgray}
			\text{ \cite{rebuffi2021fixing}}  & 88.54 & 64.46
			\\ 
            \rowcolor{lightgray}
			\text{ \cite{gowal2021improving}} & 88.74 & 66.60\\
             \rowcolor{lightgray}
			 \textbf{LoRID} ($50,10$) & 85.30 & {69.34}\\
         \rowcolor{lightgray}
			 \textbf{LoRID} ($60,10$) & 85.10 & \textbf{70.87}\\
                \hline
               \cite{wang2022guided} & 86.76& 37.11 \\
                \cite{nie2022DiffPure}  &  \textbf{90.07} &  45.31
			\\
            \textbf{LoRID} ($25,20$)& 84.60 & \textbf{66.40}\\
            \textbf{LoRID} ($10,40$)& 86.90 & 59.20\\
			\hline
		\end{tabular}
		\vspace{-4pt}
	\end{table}

\begin{table}[t] 
		\centering
		\caption{Standard accuracy and robust accuracy against AutoAttack $L_2$ ($\epsilon=0.5$) on CIFAR-10. * indicates the usage of extra data. The gray and white boxes indicate the black-box and white-box attacks.}
		\label{tab:L_2_cifar10}
		\footnotesize\addtolength{\tabcolsep}{0pt}
  \vspace{-4pt}
		\begin{tabular}{ccc}
			\hline
			\textbf{Method} &  {Standard Acc} & {Robust Acc} \\
			\hline & \\[-2.4ex]

		      \rowcolor{lime}
        \multicolumn{3}{c}
			{WideResNet-28-10} 
			\\
			\hline 
            \rowcolor{lightgray}
			\text{ \cite{Pang2022RobustnessAA}}*  & 90.83  & 78.10 \\ 
             \rowcolor{lightgray}
			\text{ \cite{rebuffi2021fixing}}*  & \textbf{91.79}  & 78.69 \\ 
            \rowcolor{lightgray}
			\text{ \cite{wang2023better}}*  & \textbf{95.16}  & 83.68 \\ 
           \rowcolor{lightgray}
        			 \textbf{LoRID} ($39,4$) & 90.34 & \textbf{89.69}  \\ 
            \hline
            \text{ \cite{wang2022guided}}  & 85.66 & 73.32 \\ 
                \cite{nie2022DiffPure}  &  91.03 &  64.06
    			\\
                \textbf{LoRID} ($15,30$) & 85.4 & \textbf{77.9} \\
                \textbf{LoRID} ($20,24$) & 84.2 & 73.6\\
\hline & \\[-2.4ex]
		\end{tabular}
		\vspace{-8pt}
	\end{table}

\paragraph{CIFAR-10.} Tables~\ref{tab:L_inf_cifar10} and~\ref{tab:L_2_cifar10} show the defense's performance under $L_\infty (\epsilon = 8/255)$ and $L_2 (\epsilon = 0.5)$ AutoAttack on CIFAR-10. Our method achieves significant improvements in both standard and robust accuracy compared to previous SOTA in both black-box and white-box settings. Particularlly, LoRID improves black-box robust accuracy by 23.15\% on WideResNet-28-10 and by 4.27\% on WideResNet-70-16. Additionally, our method surpasses baseline robust accuracy in the white-box by 12.26\% on WideResNet-28-10 and by 21.09\% on WideResNet-70-16.

\begin{table}[t] 
		\centering
		\caption{Standard accuracy and robust accuracy against white-box PGD+EOT $L_\infty$  ($\epsilon={4}/{255}$) on ImageNet.}
		\label{tab:imagenet}
		\footnotesize\addtolength{\tabcolsep}{-2pt}
		\vspace{-4pt}
		\begin{tabular}{ccc}
			\hline
			\textbf{Method} &  {Standard Acc} & {Robust Acc} \\
			\hline & \\[-2.4ex]

		      \rowcolor{lime}
        \multicolumn{3}{c}
			{WideResNet-28-10} 
			\\
			\hline 
   \text{ \cite{wong2020fast}} & 53.83 & 28.04 \\ 
            \text{ \cite{robustness}} & 62.42 & 33.20 \\ 
			\text{ \cite{salman2020adversarially}} & 68.46 & 39.25 \\ 
                \cite{nie2022DiffPure} & 71.16  & 44.39
    			\\
                \cite{lee2023robust} & 70.74 & 42.15 \\
                \textbf{LoRID} ($5,30$) & \textbf{73.98} & \textbf{56.54 }
			\\
			\hline & \\[-2.4ex]
		\end{tabular}
		\vspace{-4pt}
	\end{table}


\paragraph{ImageNet.} Table~\ref{tab:imagenet} shows the robustness performance against $L_\infty (\epsilon = 4/255)$ AutoAttack on WideResNet-28-10. Our method significantly outperforms SOTA baselines in both standard and robust accuracies.


\begin{table}[!ht] 
		\centering
		\caption{Standard accuracy and robust accuracy against BPDA+EOT $L_\infty$  on  Celeb HQ-Eyeglasses attribute classifier, with  $\epsilon = {16}/{255}$.}
		\label{tab:celeb}
		\footnotesize\addtolength{\tabcolsep}{-0pt}
		\vspace{-4pt}
		\begin{tabular}{ccc}
			\hline
			\textbf{Method} &  {Standard Acc} & {Robust Acc} \\

			\hline & \\[-2.4ex]
			\rowcolor{lime}
   \multicolumn{3}{c}
			{Eyeglasses attribute classifier for CelebA-HQ}  \\
			\hline
            \rowcolor{lightgray}
            \cite{chai2021ensembling} & 99.37 & 26.37 \\
            \rowcolor{lightgray}
			\cite{richardson2021encoding} & 93.95 & 75.00 \\
            \rowcolor{lightgray}
                \cite{nie2022DiffPure}  &  93.77 & 90.63
			\\
            \rowcolor{lightgray}
            \textbf{LoRID} ($100,15$) & 98.91 &{97.80} \\
			\hline

		\end{tabular}
		\vspace{-8pt}
	\end{table}

\paragraph{CelebA-HQ.} For large datasets like CelebA-HQ, attackers often use the BPDA+EOT attack~\cite{tramer2020adaptive,hill2021stochastic}, which substitutes exact gradients with classifier gradients. We evaluated our approach against baseline methods under this attack, as shown in Table~\ref{tab:celeb}. Our method outperforms the best baseline in robust accuracy by 7.17\%, while also maintaining high standard accuracy.






 \begin{table}[!ht] 
		\centering
		\caption{Performance of LoRID against black-box AutoAttack on CIFAR-10 at high-noise regime.}
		\label{tab:L_inf_lorid_cifar10}
		\footnotesize\addtolength{\tabcolsep}{0pt}
		\vspace{-4pt}
		\begin{tabular}{cccc}
			\hline
			\textbf{Method} & {$\epsilon$} & {Standard Acc} & {Robust Acc} \\
			\hline & \\[-2.4ex]

		      \rowcolor{lime}
        \multicolumn{4}{c}
			{WideResNet-28-1 $L_\infty$ attacks } 
			\\
			\hline 
            \rowcolor{lightgray}
			\text{ \cite{gowal2021improving}}  &$8/255$  & 87.50 & 65.24\\ 
                \rowcolor{lightgray} \textbf{LoRID} ($39,5$) & $8/255$ & 90.41 & \textbf{88.39} \\ 
                \rowcolor{lightgray} \textbf{LoRID} ($\textsc{TF},40,2$) & $8/255$ & 89.32 & 88.12 \\ \hline
                \rowcolor{lightgray} \textbf{LoRID} ($49,8$) & $16/255$ & 89.00 & 85.86 \\
                \rowcolor{lightgray} \textbf{LoRID} ($\textsc{TF},42,5$) & $16/255$ & 88.66 & \textbf{86.23} \\ \hline
                \rowcolor{lightgray} \textbf{LoRID} ($49,12$) & $32/255$ & 89.20 & 69.87 \\  
                \rowcolor{lightgray} \textbf{LoRID} ($\textsc{TF},48,9$) & $32/255$ & 88.35 & \textbf{78.04} \\
                \hline & \\[-2.4ex]

		\end{tabular}
		\vspace{-12pt}
	\end{table}

 \paragraph{High-noise regime.} We demonstrate the effectiveness of Tucker decomposition in high-noise settings, as shown in Table~\ref{tab:L_inf_lorid_cifar10}. Specifically, we compare LoRID to the best known robustness results from \citet{gowal2021improving} under black-box $L_\infty$ AutoAttack. The results indicate that Tucker decomposition becomes increasingly beneficial as noise levels rise, as supported by Theorem~\ref{appx:theorem:ddpm_tf}. Notably, with Tucker decomposition, LoRID's robustness at a very high noise level ($\epsilon = 32/255$) surpasses SOTA performance at the standard noise level ($\epsilon = 8/255$) by 12.8\%.

\section{Conclusion} \label{sect:conclusion}


We introduced LoRID, a defense strategy that uses multiple looping in the early stages of diffusion models to purify adversarial examples. To enhance robustness in high noise regimes, we integrated Tucker decomposition. Our approach, validated by theoretical analysis and extensive experiments on CIFAR-10/100, ImageNet, and CelebA-HQ, significantly outperforms state-of-the-art methods against strong adaptive attacks like AutoAttack, PGD+EOT and BPDA+EOT. 

\bibliography{main}


\clearpage

\appendix

\section{Appendix}

\subsection{Proof of Theorem~\ref{theorem:ddpm_converge}} \label{appx:theorem:ddpm_converge}

In the following, we provide the proof of Theorem~\ref{theorem:ddpm_converge}. We restate the Theorem below:
\begin{theorem*}
Let $\left\{\bmx^{(i)}_t \right\}_{t \in \{0,...,T \}},  i \in \{1,2\}$ be two diffusion processes given by the forward equation (\ref{eq:forward}) of a DDPM. Denote $q^{(1)}_t$ and $q^{(2)}_t$ the distributions of $\bmx_t^{(1)}$ and $\bmx_t^{(2)}$, respectively. Then, for all $t \in \{0,..., T-1\}$, we have
\begin{align*}
    D_{KL}\left(q^{(1)}_t||q^{(2)}_t \right) \geq D_{KL}\left(q^{(1)}_{t+1}||q^{(2)}_{t+1} \right)
\end{align*}
\end{theorem*}
\begin{proof}
We start by restating the \textit{chain rule for relative entropy}~\cite{Cover2006}:
\begin{align}
    & D_{KL}\left(p(\bmz_1,\bmz_2) ||p'(\bmz_1,\bmz_2) \right) \nonumber \\ =& D_{KL}\left(p(\bmz_1)||p'(\bmz_1) \right)  
    + D_{KL}\left(p(\bmz_2| \bmz_1)||p'(\bmz_2 | \bmz_1) \right)
\end{align}
Then, by denoting $q^{(i)}(\bmx_{t+1}, \bmx_t)$ the joint distribution of $\bmx^{(i)}_{t+1}$ and $ \bmx^{(i)}_t$, the chain rule gives us:
\begin{small}
\begin{align}
     &D_{KL}\left( q^{(1)}(\bmx_{t+1}, \bmx_t) ||q^{(2)}(\bmx_{t+1}, \bmx_t) \right)  \\
     = \ & D_{KL}\left( q^{(1)}_{t+1} ||q^{(2)}_{t+1} \right) + D_{KL}\left( q^{(1)}(\bmx_{t}| \bmx_{t+1}) ||q^{(2)}(\bmx_{t} | \bmx_{t+1}) \right) \\
     = \ & D_{KL}\left( q^{(1)}_{t} ||q^{(2)}_{t} \right) + D_{KL}\left( q^{(1)}(\bmx_{t+1}| \bmx_{t}) ||q^{(2)}(\bmx_{t+1} | \bmx_{t}) \right) \label{eq:KLchainrule}
\end{align}
\end{small}
Note that, due to (\ref{eq:forward}), we have $ q^{(1)}(\bmx_{t+1}| \bmx_{t}) = q^{(2)}(\bmx_{t+1} | \bmx_{t})$, this implies the last term of Eq. (\ref{eq:KLchainrule}) $D_{KL}\left( q^{(1)}(\bmx_{t+1}| \bmx_{t}) ||q^{(2)}(\bmx_{t+1} | \bmx_{t}) \right) =0 $. Thus, we have:
\begin{small}
\begin{align}
    & D_{KL}\left( q^{(1)}_{t+1} ||q^{(2)}_{t+1} \right) + D_{KL}\left( q^{(1)}(\bmx_{t}| \bmx_{t+1}) ||q^{(2)}(\bmx_{t} | \bmx_{t+1}) \right) \nonumber \\
     &=  D_{KL}\left( q^{(1)}_{t} ||q^{(2)}_{t} \right) 
\end{align}
\end{small}
Thus, due to the non-negativity of the KL divergence $D_{KL}\left( q^{(1)}(\bmx_{t}| \bmx_{t+1}) ||q^{(2)}(\bmx_{t} | \bmx_{t+1}) \right) $, we have the Theorem:
\begin{align}
    D_{KL}\left(q^{(1)}_t||q^{(2)}_t \right) \geq D_{KL}\left(q^{(1)}_{t+1}||q^{(2)}_{t+1} \right)
\end{align}
\end{proof}

\subsection{Proof of Theorem~\ref{theorem:ddpm_time}} \label{appx:theorem:ddpm_time}

\begin{theorem*} 
Let $\left\{\bmx_t \right\}_{t \in \{0,...,T \}}$ be a diffusion process defined by the forward equation (\ref{eq:forward}) where $\bmx_0$ is the adversarial sample. i.e, $\bmx_0 = \bmx_{clean} + \boldsymbol{\epsilon}_a$. For any given time $t$, we have
\begin{align*}
    \mathbb{E} \left [\|\hat{\bmx}_0(t) - \bmx_{clean} \| \right]\geq  \textup{MMSE}\left( \frac{{\bar{\alpha}_t}}{{1- \bar{\alpha}_t}}  \right) - \| \boldsymbol{\epsilon}_a \|  
\end{align*}
where the expectation is taken over the distribution of $\bmx_{clean}$ and $\textup{MMSE}(\textsc{SNR})$ is the minimum mean-square error achievable by optimal estimation of the input given the output of Gaussian channel with a signal-to-noise ratio of $\textsc{SNR}$. The function $\textup{MMSE}(\textsc{SNR})$ has the following form~\cite{mmse_mi}:
\begin{align*}
1 - \frac{1}{\sqrt{2\pi}} \int_{-\infty}^{\infty} e^{-y^2/2 }\textup{tanh} \left( \textsc{SNR} - \sqrt{\textsc{SNR}} y
\right) \textup{d} y
\end{align*}
\end{theorem*}

\begin{proof}
    We consider (\ref{eq:diff_xt}) as a Gaussian channel $\mathbf{y}_t  = \frac{\sqrt{\bar{\alpha}_t}}{\sqrt{1- \bar{\alpha}_t}} \bmx_0 + \boldsymbol{\epsilon}_0$, and $ \hat{\bmx}_0(t)$ as an estimation of ${\bmx}_0(t) $ given $\mathbf{y}_t$. By denoting $\hat{\bmx}^*_0( \mathbf{y}_t)$ the best estimator of $\bmx_0$ given $\mathbf{y}_t$, we have    
\begin{align}
     \mathbb{E} \left[ 
    \left \| 
   \hat{\bmx}_0(t)  - \bmx_0
    \right \|
    \right] \geq& \mathbb{E} \left[ 
    \left \| 
   \hat{\bmx}^*_0( \mathbf{y}_t)  - \bmx_0
    \right \|
    \right]  \nonumber \\
    =&  \textup{MMSE}\left( \frac{{\bar{\alpha}_t}}{{1- \bar{\alpha}_t}}  \right)  \label{eq:bound_for_px}
\end{align}
where the equality is from the definition of the $\textup{MMSE}(\textsc{SNR})$ function. We are now ready to show (\ref{eq:error_bound}). In fact, from the triangle inequality, we have:
 \begin{align*}
  \|\hat{\bmx}_0(t) - \bmx_{clean} \| &= \|\hat{\bmx}_0(t) -  \bmx_{0} -\boldsymbol{\epsilon}_a  \| \\
  &\geq   \|\hat{\bmx}_0(t) -  \bmx_{0}\| - \| \boldsymbol{\epsilon}_a  \| 
 \end{align*}
 Combining the above with (\ref{eq:bound_for_px}) gives us:
 \begin{align}
     \mathbb{E} \left[ 
    \left \| 
   \hat{\bmx}_0(t)  - \bmx_{clean}
    \right \|
    \right] \geq  \textup{MMSE}\left( \frac{{\bar{\alpha}_t}}{{1- \bar{\alpha}_t}}  \right) - \| \boldsymbol{\epsilon}_a  \|
 \end{align}
 Thus, we have the Theorem. 
 
 For comprehensiveness, we now highlight how to derive (\ref{eq:mmse}). Particularly, we use the following relation between mutual information of a channel, i.e., $I(\textsc{snr})$, and  the minimum mean-square error of Gaussian channel~\cite{mmse_mi}:
\begin{align}
    \frac{\textup{d}{I}(\textsc{snr})}{\textup{d}\textsc{snr}} = \frac{1}{2}{\textup{MMSE}(\textsc{snr})} 
\end{align}
where the mutual information of a channel $I(\textsc{SNR})$ is given as~\cite{Blahut_mi_principle} (p. 274), and~\cite{Gallager_info_theory} (Problem 4.22):
\begin{small}
\begin{align}
 \textsc{SNR} - \frac{1}{\sqrt{2\pi}} \int_{-\infty}^{\infty} e^{-y^2/2 }\log \textup{cosh} \left( \textsc{SNR} - \sqrt{\textsc{SNR}} y
\right)  \textup{d} y.
\label{eq:mi}
 \end{align}
\end{small}Here, the mutual information is computed in nats. Taking the derivative of (\ref{eq:mi}) gives us (\ref{eq:mmse}).
\end{proof}


\subsection{Proof of Theorem~\ref{theorem:ddpm_time_upper}} \label{appx:theorem:ddpm_time_upper}
\begin{theorem*}
Additionally to the conditions stated in Theorem~\ref{theorem:ddpm_time}, if the DDPM is able to recover the original signal $\bmx_0 $ within an error $\boldsymbol{\delta}_{\textsc{DDPM}}(t)$ in the expectation, i.e., for all $t$,
\begin{align*}
     \mathbb{E} \left[ 
    \left \| 
   \hat{\bmx}_0(t)  - \bmx_0
    \right \|
    \right] \leq \mathbb{E} \left[ 
    \left \| 
   \hat{\bmx}^*_0( \mathbf{y}_t)  - \bmx_0
    \right \|
    \right]  + \boldsymbol{\delta}_{\textsc{DDPM}}(t) 
\end{align*}
 where $\hat{\bmx}^*_0( \mathbf{y}_t) $ is  the best estimator of $\bmx_0$ given $\mathbf{y}_t  = \frac{\sqrt{\bar{\alpha}_t}}{\sqrt{1- \bar{\alpha}_t}} \bmx_0 + \boldsymbol{\epsilon}_0$, then, we have the reconstructed error $\mathbb{E} \left [\|\hat{\bmx}_0(t) - \bmx_{clean} \| \right]$ is upper-bounded by:
\begin{align*}
     \textup{MMSE}\left( \frac{{\bar{\alpha}_t}}{{1- \bar{\alpha}_t}}  \right) + \boldsymbol{\delta}_{\textsc{DDPM}}(t) + \| \boldsymbol{\epsilon}_a \| 
\end{align*}
\end{theorem*}
\begin{proof}
    From the triangle inequality, we have:
 \begin{align}
  \|\hat{\bmx}_0(t) - \bmx_{clean} \| =& \|\hat{\bmx}_0(t) -  \bmx_{0} -\boldsymbol{\epsilon}_a  \| \nonumber \\
  \leq&   \|\hat{\bmx}_0(t) -  \bmx_{0}\| + \| \boldsymbol{\epsilon}_a  \| \nonumber
 \end{align}
 Combining the above with the assumption on the recovering error stated in the Theorem gives us:
 \begin{align*}
    & \mathbb{E} \left[ 
    \left \| 
   \hat{\bmx}_0(t)  - \bmx_0
    \right \| + \| \boldsymbol{\epsilon}_a  \| 
    \right] =  \
    \mathbb{E} \left[ 
    \left \| 
   \hat{\bmx}_0(t)  - \bmx_0
    \right \| \right]+  \mathbb{E} \left[  \| \boldsymbol{\epsilon}_a  \| 
    \right]
    \\
    \leq & \
    \mathbb{E} \left[ 
    \left \| 
   \hat{\bmx}^*_0( \mathbf{y}_t)  - \bmx_0
    \right \|
    \right]  + \boldsymbol{\delta}_{\textsc{DDPM}}(t)  +  \| \boldsymbol{\epsilon}_a  \| \\
    = &
     \textup{MMSE}\left( \frac{{\bar{\alpha}_t}}{{1- \bar{\alpha}_t}}  \right) + \boldsymbol{\delta}_{\textsc{DDPM}}(t)  +  \| \boldsymbol{\epsilon}_a  \|
 \end{align*}
 We then have the Theorem.
\end{proof}

\subsection{Proof of Theorem~\ref{theorem:ddpm_time_loop}} \label{appx:theorem:ddpm_time_loop}

\begin{theorem*} 
Let $\left\{\bmx_t \right\}_{t \in \{0,...,T \}}$ be a diffusion process defined by the forward Eq. (\ref{eq:forward}) where $\bmx_0$ is the adversarial sample. i.e, $\bmx_0 = \bmx_{clean} + \boldsymbol{\epsilon}_a$. For any given time $t$, we have the reconstructed error $ \mathbb{E} \left [\|\hat{\bmx}^L_0(t/L) - \bmx_{clean} \| \right] $ is upper-bouned by:
\begin{small}
\begin{align*}
      L \times \left(  \textup{MMSE}\left( \frac{{\bar{\alpha}_{t/L}}}{{1- \bar{\alpha}_{t/L}}}  \right)  + \boldsymbol{\delta}_{\textsc{DDPM}}\left( \frac{t}{L}\right)  \right)  + \| \boldsymbol{\epsilon}_a \| 
\end{align*}
\end{small}where the expectation is taken over the distribution of $\bmx_{clean}$.
\end{theorem*}

\begin{proof}
From the triangle inequality, we have:
\begin{small}
 \begin{align}
  &\|\hat{\bmx}^L_0(t/L) - \bmx_{clean} \| \nonumber \\
  =& \left \| \sum_{l=1}^{L-1} (\hat{\bmx}^{l+1}_0(t/L)-  \hat{\bmx}^{l}_0(t/L)) + (\hat{\bmx}^1_0(t/L)-  \bmx_{0}) -\boldsymbol{\epsilon}_a  \right \| \\
  \leq &  \sum_{l=1}^{L-1} \|\hat{\bmx}^{l+1}_0(t/L)-  \hat{\bmx}^{l}_0(t/L))\|  + \|\hat{\bmx}^1_0(t/L)-  \bmx_{0} \| + \| \boldsymbol{\epsilon}_a  \| \label{eq:triangle_toverl}
 \end{align}
\end{small}Noting that, each of the signal $\hat{\bmx}^{l+1}_0(t/L)$ is the output of the purification $r_{t/L} \circ f_{t/L}$ on the input $\hat{\bmx}^{l}_0(t/L)$. Thus, from the condition of $\boldsymbol{\delta}_{\textsc{DDPM}}$ (stated in Theorem~\ref{theorem:ddpm_time_upper}), we have:
\begin{align}
   &\mathbb{E} \left[ \|\hat{\bmx}^{l+1}_0(t/L)-  \hat{\bmx}^{l}_0(t/L))\|  \right]
    \nonumber \\
    \leq& \mathbb{E} \left[ 
    \left \| 
   \hat{\bmx}^{{l+1}^*}_0( \mathbf{y}^l_{t/L})  - \hat{\bmx}^{l}_0(t/L))
    \right \|
    \right]  + \boldsymbol{\delta}_{\textsc{DDPM}}(t/L), \label{eq:expectation_msse}
\end{align}
for all $t$, where $\mathbf{y}^l_{t/L} =  \frac{\sqrt{\bar{\alpha}_{t/L}}}{\sqrt{1- \bar{\alpha}_{t/L}}} \hat{\bmx}^{l}_0(t/L) + \boldsymbol{\epsilon}$, and $\hat{\bmx}^{{l+1}^*}_0( \mathbf{y}^l_{t/L})$ is the optimal reconstruction of $ \hat{\bmx}^{l}_0(t/L)$ given $\mathbf{y}^l_{t/L}$. 

Notice that the SNR of that channel is  $\frac{\sqrt{\bar{\alpha}_{t/L}}}{\sqrt{1- \bar{\alpha}_{t/L}}}$, thus, $\mathbb{E} \left[ 
    \left \| 
   \hat{\bmx}^{{l+1}^*}_0( \mathbf{y}^l_{t/L})  - \hat{\bmx}^{l}_0(t/L))
    \right \|
    \right] =  \textup{MMSE}\left( \frac{{\bar{\alpha}_{t/L}}}{{1- \bar{\alpha}_{t/L}}}  \right)$. Applying that to (\ref{eq:expectation_msse}) gives us
\begin{align}
     &\sum_{l=1}^{L-1} \|\hat{\bmx}^{l+1}_0(t/L)-  \hat{\bmx}^{l}_0(t/L))\|  + \|\hat{\bmx}^1_0(t/L)-  \bmx_{0} \| \nonumber \\
     \leq & L \times \left(  \textup{MMSE}\left( \frac{{\bar{\alpha}_{t/L}}}{{1- \bar{\alpha}_{t/L}}}  \right) + \boldsymbol{\delta}_{\textsc{DDPM}}(t/L) \right),
\end{align}
since $\hat{\bmx}^1_0(t/L) $ can also be considered as the optimal reconstruction of $ \bmx_{0}$ given $\mathbf{y}^1_{t/L} =  \frac{\sqrt{\bar{\alpha}_{t/L}}}{\sqrt{1- \bar{\alpha}_{t/L}}} {\bmx}_0+ \boldsymbol{\epsilon}$. Using the above result on  (\ref{eq:triangle_toverl}) gives us the Theorem.
\end{proof}

\subsection{Proof of Theorem~\ref{theorem:ddpm_tf}} \label{appx:theorem:ddpm_tf}
\begin{theorem*} 
    The reconstruction errors introduced of the purification $ r_{t}  \circ f_{t} \circ \textsc{TF}$ is bounded by:
    \begin{align*}
    \textsc{MSE}_{r_{t}  \circ f_{t} \circ \textsc{TF} }(\boldsymbol{\epsilon}_a) &\leq \textup{MMSE}\left( \frac{{\bar{\alpha}_t}}{{1- \bar{\alpha}_t}}  \right) + \boldsymbol{\delta}_{\textsc{DDPM}}(t) \\
    &+ \textsc{E}_{\textsc{Tucker}} + \left \| \textsc{TF}(\boldsymbol{\epsilon}_a ) \right \| 
\end{align*}
\end{theorem*}

\begin{proof}
We denote the purified signal of the adaptation of tensor-factorization into the purification process of DDPM by:
\begin{align}
    \hat{\bmx}_{\textsc{DDPM-TF}}:= r_{t}  \circ f_{t} \circ \textsc{TF} \ (\bmx + \boldsymbol{\epsilon}_{a}) \label{eq:pur_tfddpm}
\end{align}
By denoting $\boldsymbol{\epsilon}^{\textsc{TF}}_a :=  \textsc{TF}(\bmx + \boldsymbol{\epsilon}_a) - \bmx $, we can consider (\ref{eq:pur_tfddpm}) as the applying of the purification $r_{t}  \circ f_{t}$ onto $\bmx + \boldsymbol{\epsilon}^{\textsc{TF}}_a$. By applying Theorem~\ref{theorem:ddpm_time_upper}, we have:
\begin{small}
\begin{align}
    \textsc{MSE}_{r_{t}  \circ f_{t} \circ \textsc{TF} }(\boldsymbol{\epsilon}_a) &\leq \textup{MMSE}\left( \frac{{\bar{\alpha}_t}}{{1- \bar{\alpha}_t}}  \right) + \boldsymbol{\delta}_{\textsc{DDPM}}(t) + \| \boldsymbol{\epsilon}^{\textsc{TF}}_a \| \label{eq:error_bound_tfddpm_1}
\end{align}
\end{small}Since $\boldsymbol{\epsilon}^{\textsc{TF}}_a$ is the purification error of tensor-factoriazation on $\bmx + \boldsymbol{\epsilon}_a$, (\ref{eq:error_bound_tf}) implies:
\begin{align}
    \| \boldsymbol{\epsilon}^{\textsc{TF}}_a \| = \textsc{MSE}_{\textsc{TF}}(\boldsymbol{\epsilon}_a) \leq  \textsc{E}_{\textsc{Tucker}} + \left \| \textsc{TF}(\boldsymbol{\epsilon}_a ) \right \|
\end{align}
which gives us
\begin{align}
    \textsc{MSE}_{r_{t}  \circ f_{t} \circ \textsc{TF} }(\boldsymbol{\epsilon}_a) \leq& \textup{MMSE}\left( \frac{{\bar{\alpha}_t}}{{1- \bar{\alpha}_t}}  \right) + \boldsymbol{\delta}_{\textsc{DDPM}}(t) \nonumber
    \\
    +& \textsc{E}_{\textsc{Tucker}} + \left \| \textsc{TF}(\boldsymbol{\epsilon}_a ) \right \| 
\end{align}
\end{proof}
\section{Experimental Details}\label{append:Implmentation}
In this appendix, we provide the details of the experimental results reported in our main manuscript.

\subsection{Experimental details of Fig.~\ref{fig:msse_vs_max_iter}} \label{appx:mmse}

Intuitively, Theorem~\ref{theorem:ddpm_time_loop} captures the impact of increasing the iterative factor $L$ on the reconstruction error induced by the purification process. Experiment in Fig.~\ref{fig:msse_vs_max_iter} aims to illustrate that behavior. As the MMSE (\ref{eq:mmse}) depends on the actual input distribution, to visualize the lower bound of Theorem~\ref{theorem:ddpm_time_loop}, we consider the input of the Gaussian channel $\mathbf{y}_t  = \frac{\sqrt{\bar{\alpha}_t}}{\sqrt{1- \bar{\alpha}_t}} \bmx_0 + \boldsymbol{\epsilon}_0$ induced by (\ref{eq:diff_xt}) to be a Gaussian signal. Thus, the $\textup{MMSE}(\textsc{SNR})$ is simply $1/(1+SNR)$~\cite{mmse_mi}. Given that, we can plot the dominent term $  L \times \left(  \textup{MMSE}\left( \frac{{\bar{\alpha}_{t/L}}}{{1- \bar{\alpha}_{t/L}}}  \right) \right) $ of the bound in Theorem~\ref{theorem:ddpm_time_loop} and plot it in Fig.~\ref{fig:msse_vs_max_iter}.

\subsection{Tucker Tensor Decomposition} \label{appx:tucker}
Formally, given a tensor \(\mathcal{X}\) of size $I_1 \times I_2 \times ... \times I_N$, we project it into a lower-dimensional space using Tucker factor matrices: $\mathcal{G} = \mathcal{X} \times_{2} \bm U^{T}_2 \times_{3} \bm U^T_{3} \times_{4} \bm U^{T}_{4} \times_{5} \bm U^{T}_{5}$, where $\mathcal{G}$ is the core tensor of size $r_1 \times r_2 \times ... \times r_5$ for $r_n \leq I_n$ $(n \in \{1,2, ..,5\})$, $\bm U_n$ is the factor matrix for mode $n$ with size $I_n \times r_n$, and $\times_n$ denotes the mode-$n$ product. The reconstruction signal is $\hat{\mathcal{X}} = \mathcal{G} \times_{5} \bm U_{5} \times_{4} \bm U_{4} \times_{3} \bm U_{3} \times_{2} \bm U_{2}$. For convenient, we express the above process of projecting-recovering as $ \mathcal{X} \approx  \hat{\mathcal{X}} = d^{-1}  \circ d (\mathcal{X})$, where $d$ and $d^{-1}$ denote the projection and the reconstruction, respectively.

\subsection{Attacking methods} \label{appx:attack_detail}
We use AutoAttack at perturbation levels of $(\frac{8}{255}, \frac{16}{255}, \frac{32}{255})$ and when applicable, compare our results with the reported accuracies in the RobustBench benchmark. For evaluation, we employ the RobustBench codebase and model zoo\footnote{\url{ https://github.com/RobustBench/robustbench}} to obtain hyperparameters whenever the standard model is available. For AutoAttack under the $L_\infty$ norm, and due to the complexity of ImageNet, we use a perturbation level of $\frac{4}{255}$ for the implementation of the EOT-PGD attack~\cite{kim2020torchattacks}. For black-box attacks, we use the standard version of AutoAttack, which includes APGD-CE, APGD-T, FAB-T, and Square, as well as PGD+EOT on CelebA-HQ. For white-box attacks, we use the AutoAttack RAND version, which comprises APGD-CE and APGD-DLR. For high iterations, we utilize the repository from~\cite{kang2024diffattack} to explicitly compute the attack gradients rather than relying on the computational graph.

\subsection{Implicit trick for surrogate gradients}\label{appx:attacking_surrogate}

To obtain the exact gradients resulting from the DDPM with $t$ effective time-steps, the attacker needs to store and backward a computational graph whose size is proportional to $N \times t \times \| \Phi \|$, where $N$ is the batch size and $\| \Phi \| $ is the size of the DDPM's denoiser. As $\Phi$ is typically of millions parameters, exact gradients' computation creates an extreme burden on computational resource. To alleviate this challenge, \citet{lee2023robust} proposes to compute the gradient in a \textit{skipping} manner. In particular, instead of iteratively compute the reconstructed signal, they use a proxy process of skipping $k$ time-step per iteration and computes:
\begin{align}
     &\bmx_{t-k}
    \approx  \frac{{\sqrt{\bar{\alpha}_{t-k}} } }{\sqrt{\bar{\alpha}_{t}}} \bmx_{t } \nonumber \\
    &+ \sqrt{\bar{\alpha}_{t-k}} \times \left( \sqrt{\frac{1 - \bar{\alpha}_{t-k}}{\bar{\alpha}_{t-k}}} 
- \sqrt{\frac{1 - \bar{\alpha}_{t}}{\bar{\alpha}_{t}}} 
\right)  \boldsymbol{\epsilon}_{\theta}\left( \bmx_t, t
    \right)
\end{align}
Then, the surrogate gradient, which is the gradients computing based on the reconstructed signal resulted from the skipping computation above, is used instead of the exact gradients. As this skipping trick reduces the computational graph by a factor of $k$, it allows attacks with reasonable computational complexity. 

It it noteworthy to point out that, this surrogate method results in attackers that is heuristically significantly stronger than the adjoint~\cite{nie2022DiffPure}, as reported in~\cite{lee2023robust}.

\subsection{Algorithm}
This appendix provides the pseudocode for LoRID, which is shown in Algo.~\ref{alg:LoRID}.
\label{appx:algorithm}
\begin{algorithm}[!ht] 
    \SetAlgoLined
    \SetKwInOut{Input}{Input}
    \SetKwInOut{Output}{Output}
    \SetKwInOut{Given}{Given}
    \Input{Input image $\bmx$, looping parameter $L$, and total purified time-step $t$}
    \Given{DDPM's forwarding and reversing functions $\{f_t,r_t\}_{t=1}^T$, and Tucker decomposition $\textsc{TF}$ }
    \Output{Purified image $\hat{\bmx}_0$}
    \textbf{Step 1: }\textit{Tucker decomposition}\\
    \Indp $\bmx \leftarrow \textsc{TF}(\bmx)$\\
    \Indm\textbf{Step 2: }\textit{Iterative Diffusion}\\
    \Indp $t' = \lfloor t/L\rfloor$ \\
    For {$l = 1$ to $L$} do: \\
    \Indp  $\bmx \leftarrow f_{t'}(r_{t'}(\bmx))$\\
    \Indm Return $\hat{\bmx}_0 \leftarrow \bmx$
     \caption{Low-Rank Iterative Diffusion}\label{alg:LoRID}
    \end{algorithm}
\subsection{Selection of LoRiD parameters} \label{appx:calibration}
In the evaluation of our model, we meticulously selected the parameters \( t \) and iterations ($L$) based on a detailed analysis of clean accuracy as a function of the LoRID parameters. The contour map presented in Figure~\ref{fig:heatmap_T_iter} illustrates the nuanced relationship between these parameters and clean accuracy, where variations in \( t \) and $L$ can lead to significant shifts in performance. Specifically, the map reveals that while selecting parameter regions that correspond to higher clean accuracy might intuitively seem advantageous, it paradoxically leads to a degradation in robust accuracy. This is attributed to the reduction in model complexity when parameters are set to optimize for clean accuracy, thereby compromising the model's robustness under adversarial conditions. Conversely, increasing \( t \) and iter enhances the complexity of the diffusion process, which bolsters robustness but at the expense of clean accuracy. Given this trade-off, our parameter selection strategy focused on identifying a balance that optimizes both clean and robust accuracies, ensuring that the model remains resilient without sacrificing performance on clean data.

    \begin{figure}
        \centering
        \includegraphics[width=0.8\linewidth]{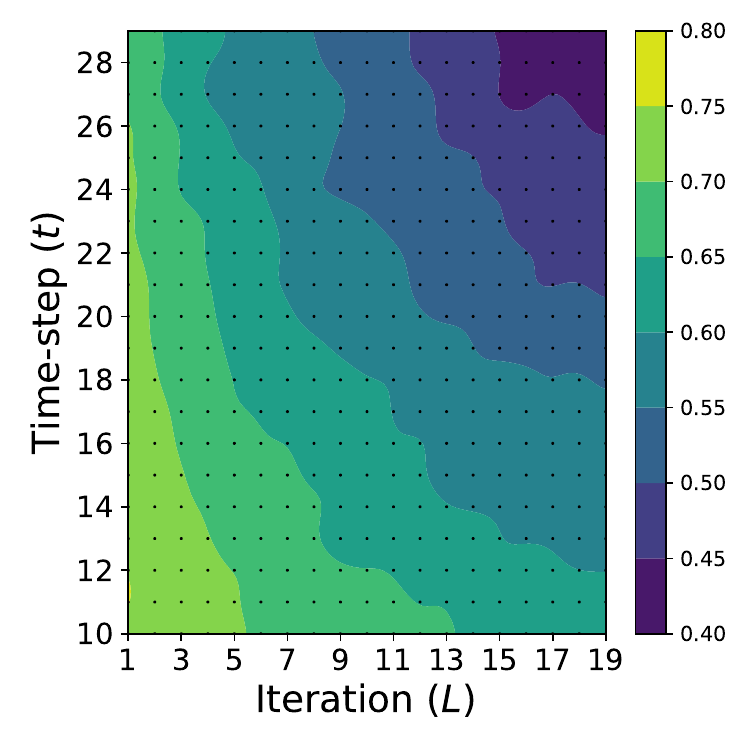}
        \caption{Impact of the time-step $t$ and iterative factor $L$ on the standard accuracy of WideResnet-28-10 in CIFAR-100 dataset.}
        \label{fig:heatmap_T_iter}
    \end{figure}
\subsection{Hardware Setup}
For this paper, our experiments utilized a HPC cluster where each node integrates four NVIDIA Hopper (H100) GPUs, each paired with a corresponding NVIDIA Grace CPU via NVLink-C2C, facilitating rapid data transfer crucial for intensive computational tasks. The GPUs are equipped with 96GB of HBM2 memory, optimal for handling large models and datasets. This setup is supported by an HPE/Cray Slingshot 11 interconnect with a bandwidth of 200GB/s, ensuring efficient inter-node communication essential for scalable machine learning operations.

\subsection{Code availability}
The code used in this study is currently under review for release by the organization. We are awaiting approval, and once granted, the code will be made publicly available.

\end{document}